\documentclass[runningheads,a4paper]{llncs}

\usepackage{graphicx}
\usepackage[all, knot]{xy}
\usepackage{amssymb,amsmath,latexsym}
\usepackage{times}
 \usepackage[usenames]{color} 
\newtheorem{lem}{Lemma}
\newcommand{\tuple}[1]{\langle #1 \rangle}  
\newcommand{\att}{\rightarrow}  
  
\newcommand{\AR}{\mathcal{A}} 
\newcommand{\attack}[2]{#1\rightarrow #2}  
\newcommand{\attackd}[2]{#1\rightarrow^\mathrm{d} #2}  
\newcommand{\attackees}[1]{#1^\rightarrow} 
\newcommand{\attackers}[1]{#1^\leftarrow} 
\newcommand{\dg}[1]{\mathrm{d}(#1)} 
\newcommand{\dgn}[1]{#1^\mathrm{d}} 
\newcommand{\nmd}[1]{#1^\textsc{def}} 
\newcommand{\und}[1]{#1^{\textsc{d}\mathrm{o}\textsc{d}}} 

\usepackage{url}

\newcommand{\keywords}[1]{\par\addvspace\baselineskip
\noindent\keywordname\enspace\ignorespaces#1}

\begin{document}

\mainmatter  % start of an individual contribution

\title{Defense semantics of argumentation: encoding reasons for accepting arguments} 

\titlerunning{Defense semantics of argumentation}

\author{Beishui Liao\inst{1} and Leendert van der Torre\inst{2} } 
%\author{Leendert van der Torre\inst{2} } 
%
\authorrunning{B. Liao and L. van der Torre} 

\institute{
Zhejiang University, China
\and
University of Luxembourg, Luxembourg
%\url{http://www.springer.com/lncs}
}


\toctitle{Lecture Notes in Computer Science}
\tocauthor{Authors' Instructions}
\maketitle

\begin{abstract}
In this paper we show how the defense relation among abstract arguments can be used to encode the reasons for accepting arguments.  After introducing a novel notion of defenses and defense graphs,  we propose a defense semantics together with a new notion of defense equivalence of argument graphs, and compare defense equivalence with standard equivalence and strong equivalence, respectively. Then, based on defense semantics, we define two kinds of reasons for accepting arguments, i.e., direct reasons and root reasons, and a notion of root equivalence of argument graphs. Finally, we show how the notion of root equivalence can be used in argumentation summarization. 
\end{abstract}
%%Results show that under admissible, complete and preferred semantics, our semantics could be regarded as a generalization of the classical notions of semantics, and that there is a connection between our approach and the possible worlds-based approaches with regard to the probability of extensions. 

\section{Introduction}
Abstract argumentation is mainly about evaluating the status of arguments in an argument graph  \cite{DBLP:journals/ai/DunneDLW15,Baroni:KER,DBLP:journals/ai/CharwatDGWW15}, which is composed of a set of abstract arguments and a set of attacks between them \cite{DBLP:journals/ai/Dung95}. In many topics such as equivalence \cite{DBLP:conf/kr/OikarinenW10,DBLP:conf/kr/Baumann16,DBLP:conf/kr/BaumannS16}, summarization \cite{iomultipoles}, and dynamics in argumentation \cite{DBLP:journals/jair/CayrolSL10,DBLP:journals/ai/FerrettiTGES17}, the notion of extensions plays a central role.  Since in classical argumentation semantics, an extension is a set of arguments that are collectively accepted, the existing theories and approaches based on this notion are mainly focused on exploiting the status of individual arguments. However,  besides the status of individual arguments, in many situations, we need to know the reasons for accepting arguments in terms of a defense relation. The following are two simple examples. 

 \begin{picture}(206,32)
 \put(0,23){\xymatrix@C=0.5cm@R=0.1cm{
  \mathcal{F}_1: &a\ar[r]&c_1\ar[r]&c_2\ar[r]&b\ar[ld]& \mathcal{F}_2: & a\ar[r]&b\ar@/^0.0cm/[l] \\
 & &c_4\ar[lu] &c_3\ar[l] &  }}  
    
    \end{picture}

First, consider $a$, $b$ in $\mathcal{F}_1$ and $\mathcal{F}_2$. In $\mathcal{F}_1$, accepting $a$ is a reason to accept $c_2$, accepting $c_2$ is a reason to accept $c_3$, and  accepting $c_3$ is a reason to accept $a$. If we allow this relation to be transitive, we find that accepting $a$ is a reason to accept $a$. Similarly,  accepting $b$ is a reason to accept $b$. Meanwhile, in $\mathcal{F}_2$, we have: accepting $a$ is a reason to accept $a$, and accepting $b$ is a reason to accept $b$. So, from the perspective of the reasons for accepting $a$ and $b$,  $\mathcal{F}_2$ is equivalent to $\mathcal{F}_1$, or $\mathcal{F}_2$  is a summarization of  $\mathcal{F}_1$. 

Second, consider the question when two argument graphs are equivalent in a dynamic setting. For $\mathcal{F}_3$ and $\mathcal{F}_4$ below,  both of them have a complete extension $\{a, c\}$. However, the reasons of accepting $c$ in $\mathcal{F}_3$ and $\mathcal{F}_4$ are different. For the former, $c$ is defended by $a$, while for the latter,  $c$ is unattacked and has no defender. In this sense, $\mathcal{F}_3$ and $\mathcal{F}_4$ are not equivalent. For example, in order to change the status of argument $c$ from ``accepted'' to ``rejected'', in $\mathcal{F}_3$, one may produce a new argument to attack the defender $a$,  or to directly attack $c$. However, in $\mathcal{F}_4$ using an argument to attack $a$ cannot change the status of $c$, since $a$ is not a defender of $c$.

    \begin{picture}(206,22)
 \put(0,8){\xymatrix@C=0.6cm@R=0.2cm{
  \mathcal{F}_3: \hspace{0.3cm} a\ar[r]&b\ar[r]&c &\hspace{0.5cm} \mathcal{F}_4:  \hspace{0.3cm}  a\ar[r]&b&c  
   }}
    
    \end{picture}

From the above two examples, one question arises: under what conditions, can two argument graphs be viewed as equivalent? The existing notions of argumentation equivalence, including standard equivalence and strong equivalence, are not sufficient to capture the equivalence of the argument graphs in the situations mentioned above.  More specifically, $\mathcal{F}_1$ and $\mathcal{F}_2$ are not equivalent in terms of the notion of standard equivalence or that of strong equivalence, but they are equivalent in the sense that the reasons for accepting arguments $a$ and $b$ in these two graphs are the same. $\mathcal{F}_3$ and $\mathcal{F}_4$ are equivalent in terms of standard equivalence, but they are not equivalent in the sense that the reasons for accepting $c$ in these two graphs are different. Although the notion of strong equivalence can be used to identify the difference between $\mathcal{F}_3$ and $\mathcal{F}_4$, conceptually it is not defined from the perspective of reasons for accepting arguments.

Note that the reasons for accepting arguments in the above two examples are depicted in terms of a defense relation, which plays a central role in Dung's concept of admissibility and thus in admissibility based semantics. So, it is natural to define a new semantics in this paper based on a defense relation such that the reasons for accepting arguments can be encoded. 

Since the new semantics is defined at the level of abstract argumentation, it can be applied to various structured argumentations systems. In particular,  in the field of legal reasoning \cite{DBLP:books/sp/09/Bench-CaponPS09}, argumentation can be used to model legal interpretation, dialogue, and deontic reasoning, etc. In all these applications, it is useful to make clear the reasons for accepting arguments in terms of a defense relation. In this paper, we will formulate a defense semantics for abstract argumentation, while its application to various structured argumentation systems is left to future work. The structure of this paper is as follows. In Section 2, we introduce some basic notions of argumentation semantics. In Section 3, we propose the notions of defenses and defense graphs, which lay a foundation of this paper. In Section 4, we formulate defense semantics by applying classical argumentation semantics to defense graphs, and study some properties of this new semantics. In Section 5, we introduce two kinds of reasons for accepting arguments in terms of defense semantics. We conclude in Section 6.

\section{Argumentation semantics}
An argument graph or argumentation framework (AF) is defined as $\mathcal{F} = (\AR, \att)$, where $\AR$ is a finite set of arguments and $\att\subseteq \AR\times \AR$ is a set of attacks between arguments \cite{DBLP:journals/ai/Dung95}. 
%In this paper, we assume the set of arguments is finite. 

Let $\mathcal{F} = (\AR, \att)$ be an argument graph. Given a set $B\subseteq \AR$ and an argument $\alpha\in \AR$, $B$ attacks $\alpha$, denoted $\attack{B}{\alpha}$,  iff there exists $\beta\in B$ such that $\attack{\beta}{\alpha}$. 
Given an argument $\alpha\in \AR$, let  $\attackers{\alpha} = \{\beta\in \AR \mid \attack{\beta}{\alpha}\}$ be the set of arguments attacking $\alpha$, and $ \attackees{\alpha} = \{\beta\in \AR \mid \attack{\alpha}{\beta}\}$ be the set of arguments attacked by $\alpha$. When $\attackers{\alpha} = \emptyset$, we say that $\alpha$ is unattacked, or $\alpha$ is an initial argument. 

Given $\mathcal{F} = (\AR, \att)$ and $E\subseteq \AR$, we say: $E$  is \emph{conflict-free} iff $\nexists \alpha, \beta\in E$ such that $\attack{\alpha}{\beta}$; $\alpha\in \AR$ is \textit{defended} by $E$  iff $\forall{\beta}\in \alpha^\leftarrow$, it holds that $\attack{E}{\beta}$; $E$ is \textit{{admissible} } iff  $E$ is  conflict-free, and each argument in $E$ is defended by $E$; $E$ is a \textit{complete extension} iff $E$ is admissible, and each argument in $\AR$ that is defended by $E$ is in $E$; $E$ is a \textit{grounded extension} iff $E$ is the minimal (w.r.t. set-inclusion) complete extension; $E$ is a \textit{preferred extension} iff $E$ is a maximal (w.r.t. set-inclusion) complete extension; $E$ is a \textit{stable extension} iff $E$ is conflict-free and $E$ attacks each argument that is not in $E$. We use $\mathrm{\sigma}(\mathcal{F})$  to denote the set of argument extensions of $\mathcal{F}$ under semantics $\sigma$, where $\sigma$ is a function mapping each argument graph to a set of argument extensions. We use $\mathrm{co}, \mathrm{gr}, \mathrm{pr}$ and $\mathrm{st}$ to denote complete, grounded, preferred and stable semantics respectively. There are some other argumentation semantics (cf. \cite{Baroni:KER} for an overview). 

For argument graphs $\mathcal{F}_1= (\AR_1, \rightarrow_1)$  and $\mathcal{F}_2 = (\AR_2, \rightarrow_2)$, we use $\mathcal{F}_1\cup \mathcal{F}_2$ to denote $ (\AR_1\cup \AR_2, \rightarrow_1\cup \rightarrow_2)$.  The standard equivalence and strong equivalence of argument graphs are defined as follows. For simplicity, when we talk about equivalence of AFs, we mainly consider the cases under complete semantics, while the full-fledged study of equivalence will be presented in an extended version of the present paper.

\begin{definition}[Standard and strong equivalence of AFs]  \cite{DBLP:conf/kr/OikarinenW10} \label{def-sequi}
Let  $\mathcal{F}$  and $\mathcal{G}$ be two argument graphs, and $\sigma$ be a semantics.
 $\mathcal{F}$  and $\mathcal{G}$ are of standard equivalence w.r.t. a semantics $\sigma$, in symbols $\mathcal{F} \equiv^\sigma \mathcal{G}$, iff $\sigma(\mathcal{F}) = \sigma(\mathcal{G})$.
$\mathcal{F}$  and $\mathcal{G}$ are of strong equivalence w.r.t. a semantics $\sigma$, in symbols $\mathcal{F} \equiv_s^\sigma \mathcal{G}$, iff for every argument graph $\mathcal{H}$, it holds that $\sigma(\mathcal{F}\cup \mathcal{H}) = \sigma(\mathcal{G}\cup \mathcal{H})$. 
\end{definition} 

\begin{example}
Consider $\mathcal{F}_1 - \mathcal{F}_4$ in Section 1.  In terms of Definition \ref{def-sequi}, under complete semantics, since $\mathrm{co}(\mathcal{F}_1)\neq \mathrm{co}(\mathcal{F}_2)$,  $\mathcal{F}_1\not\equiv^\mathrm{co} \mathcal{F}_2$, which implies that $\mathcal{F}_1\not\equiv_s^\mathrm{co} \mathcal{F}_2$. And, since $\mathrm{co}(\mathcal{F}_3)= \mathrm{co}(\mathcal{F}_4)$,  $\mathcal{F}_3\equiv^\mathrm{co} \mathcal{F}_4$. Let $\mathcal{H}= (\{d\}, \{\attack{d}{a}\})$. Since $\mathrm{co}(\mathcal{F}_3\cup \mathcal{H}) \neq \mathrm{co}(\mathcal{F}_4\cup \mathcal{H})$, $\mathcal{F}_3\not\equiv_s^\mathrm{co} \mathcal{F}_4$. 
\end{example}

Given an argument graph $\mathcal{F} = (\AR, \rightarrow)$, the kernel of $\mathcal{F}$ under complete semantics, call \textit{c-kernel},  is defined as follows. 

\begin{definition}[c-kernel of an AF] \cite{DBLP:conf/kr/OikarinenW10}
For an argument graph $\mathcal{F} = (\AR, \rightarrow)$,  the c-kernel of  $\mathcal{F}$ is defined as $\mathcal{F}^\mathrm{ck} = (\AR, \rightarrow^\mathrm{ck})$, where
$
\rightarrow^\mathrm{ck}  =  \rightarrow \setminus \{\attack{\alpha}{\beta} \mid \alpha \neq \beta, \attack{\alpha}{\alpha}, \attack{\beta}{\beta}\}
$.
\end{definition}

According to \cite{DBLP:conf/kr/OikarinenW10}, it holds that $\mathrm{co}(\mathcal{F}) = \mathrm{co}(\mathcal{F}^\mathrm{ck})$, and for any AFs $\mathcal{F}$ and $\mathcal{G}$, $\mathcal{F}^\mathrm{ck} = \mathcal{G}^\mathrm{ck}$ iff $\mathcal{F} \equiv_s^\mathrm{co} \mathcal{G}$.

%In \cite{DBLP:conf/comma/BaumannLW16}, given a semantics $\sigma$, if for every AF  $\mathcal{F}$ it holds that $\sigma(\mathcal{F}) = \sigma(\mathcal{F}^\mathrm{ck})$, then $\sigma$ is call rational. 

\section{Defenses and defense graph}
According to classical argumentation semantics, with respect to an extension $E$, an argument $\alpha\in E$  is accepted because it is initial or for all $\gamma\in \attackers{\alpha}$, $\gamma$ is attacked by an argument in $E$. So, for all $\alpha, \beta\in E$, if there exists $\gamma\in \AR\setminus E$ such that $\attack{\alpha}{\gamma}$ and $\attack{\gamma}{\beta}$, we say that accepting $\alpha$ is a (partial) reason to accept $\beta$, denoted as $\tuple{\alpha,\beta}$. And, for all $\beta\in E$ if $\attackers{\beta}= \emptyset$  (i.e., $\beta$ is an initial argument), we say that $\beta$ is accepted without a reason, denoted as $\tuple{\o,\beta}$ where $\o$ is a symbol denoting an empty position. In this paper, $\tuple{\alpha,\beta}$ or $\tuple{\o,\beta}$ is called a \textit{defense}. 

Without referring to any specific extension, a defense $\tuple{\alpha,\beta}$ can be viewed as a relation between $\alpha$ and $\beta$ satisfying some constraints. Intuitively, there are the following two minimal constraints. First, $\{\alpha, \beta\}$ is conflict-free. Otherwise, they can not be both accepted. Second, there exists $\gamma\in \AR \setminus \{\alpha,\beta\}$ such that $\attack{\alpha}{\gamma}$ and $\attack{\gamma}{\beta}$, in the sense that $\alpha$ defends $\beta$ by attacking $\beta$'s attacker $\gamma$. Regarding the defense $\tuple{\o,\beta}$, the only constraint is that $\beta$ is initial.   

\begin{example}
Consider $\mathcal{F}_5$ below. $\tuple{\o,a}$, $\tuple{a,c}$ and $\tuple{b,d}$ are defenses. Note that the three defenses do not refer to a specific extension. % However, and say which of these are with respect to the only complete extension of the graph.

%All other pairs are not a defense. For instance, $(e, c)$ is not a defense since $\{e, c\}$ is not conflict-free, and $(\o, e)$ is not a defense since $e$ is not an initial argument. 

 \begin{picture}(206,20)
 \put(0,8){\xymatrix@C=0.5cm@R=0.4cm{
\mathcal{F}_5:  & a\ar[r] & b\ar[r]& c \ar[r]&d\\
%&&e\ar@(dl,dr)\ar[r]&f\ar[ru] %\\
%&&\tuple{\o, g}\ar[ur]&   \tuple{g,f} \ar[u] 
}}
\end{picture}
\end{example}

Based on the above analysis, we have the following definition.

\begin{definition}[Defense] \label{def-dr}
{Let $\mathcal{F} = (\AR, \att)$ be an argument graph. 
For $\alpha, \beta\in \AR$,  
$\tuple{\alpha,\beta}$ is  a defense iff $\{\alpha, \beta\}$ is conflict-free, and $\exists \gamma\in \AR$ such that $\attack{\alpha}{\gamma}$ and  $\attack{\gamma}{\beta}$;
$\tuple{\o, \beta}$ is a defense iff $\beta$ is initial.
}  
 \end{definition}

The set of defenses of $\mathcal{F}$ is denoted as $\nmd{\mathcal{F}} $.
Given a defense $\tuple{\alpha, \beta}$ or $\tuple{\o,\beta}\in  \nmd{\mathcal{F}} $,  we call $\alpha$ the \textit{defender}, and $\beta$ the \textit{defendee}, of the defense. Given a set $D\subseteq \nmd{\mathcal{F}} $,  we write $\mathrm{defendee}(D) = \{\beta \mid \tuple{\alpha,\beta}, \tuple{\o,\beta}\in D\}$ to denote the set of  defendees in $D$, $\mathrm{defender}(D) = \{\alpha  \mid  \tuple{\alpha,\beta}\in D\}$ to denote the set of  defenders in $D$, and $\mathrm{def}(D) = \mathrm{defendee}(D) \cup \mathrm{defender}(D)$ be the set of defendees and defenders in $D$. %Let $\mathcal{D} = \mathrm{def}(\dgn{\mathcal{F}} )$ be the set of defenders and defendees in $\nmd{\mathcal{F}} $. 
Note that not all arguments of an AF are included in the defenses. Consider the following example. % due to the fact some arguments can be neither in $\tuple{\alpha, \beta}$ when $\{\alpha, \beta\}$ are not conflict-free, nor in $\tuple{\o, \beta}$  when $\beta$ is self-attacked or attacked by a self-attacked argument. 
 
\begin{example}\label{ex-undercutter}
In $\mathcal{F}_6$, $\mathcal{F}_7$ and $\mathcal{F}_8$, $\tuple{a_2, a_4}$, $\tuple{a_3, a_5}$, $\tuple{a_4, a_6}$, $\tuple{\o, a_7}$, $\tuple{a_7, a_9}$,  $\tuple{\o, a_{11}}$ and $\tuple{a_{11}, a_{13}}$ are  defenses, while some defense-like pairs, for instance $(a_1, a_3)$ and  $(a_{14}, a_{13})$, are not  defenses since both $\{a_1,a_3\}$ and $\{a_{14}, a_{13}\}$ are not conflict-free. %Although $(a_1, a_3)$ is not acceptable, it hampers the acceptance of $\tuple{a_4, a_6}$ because $a_3$ in $(a_1, a_3)$ attacks $a_4$ in $\tuple{a_4, a_6}$. 
And, $(\o, a_{10})$, $(\o, a_{14})$ and $(\o, a_{15})$  are not defenses, because they are not initial arguments, but either self-attacked or attacked by a self-attacked argument.  %All of them are not acceptable, but they hamper the acceptance of other defenses. 

 
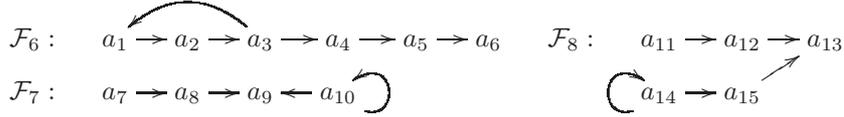
\begin{picture}(206,46)
 \put(0,26){\xymatrix@C=0.4cm@R=0.2cm{
 \mathcal{F}_6:  &a_1\ar[r] & a_2\ar[r]& a_3\ar[r]\ar@/_0.5cm/[ll]& a_4\ar[r]&a_5\ar[r]&a_6&\mathcal{F}_8:  &a_{11}\ar[r] & a_{12}\ar[r]& a_{13} \\
 \mathcal{F}_7:  & a_7\ar[r] & a_8\ar[r]& a_9 & a_{10}\ar[l]\ar@(rd,ru)&&&&a_{14}\ar@(ld,lu)\ar[r]&a_{15}\ar[ru] 
%&&\tuple{\o, g}\ar[ur]&   \tuple{g,f} \ar[u] 
}}
\end{picture}
\end{example}

%Besides $\tuple{\alpha, \beta}$ and $\tuple{\o,\beta}$, there are other pairs of arguments that do not satisfy the above-mentioned constraints. The arguments in these pairs might not be accepted, but may hamper the acceptance of other arguments. For instance, in $\mathcal{F}_5$, if without $e$ and $d$, accepting $a$ is a reason to accept $c$. However, due to the existence of $e$ and $d$, the acceptance of $c$ is hampered by them because their status is undecided. In this paper, we 

Given a defense $\tuple{x, \alpha}$ where $x \in \AR\cup \{\o\}$ and $\alpha \in \AR$, $\tuple{x, \alpha}$ can be regarded as a meta-argument. Its status is affected by other defenses and/or other defense-like pairs (cf. $(a_1, a_3)$ and $(\o, a_{10})$ in Example \ref{ex-undercutter}).  Since the pairs like  $(a_1, a_3)$ and $(\o, a_{10})$ are not accepted as a defense, but may be used to hamper the acceptance of some defenses, their behavior is similar  to that of \textit{defeaters} in defeasible logic. We call them \textit{defeaters of defenses (DoD)}.

%As illustrated in Example \ref{ex-undercutter},  some pairs of arguments, like $(a_1, a_3)$ and $(\o, a_{10})$,  might not be acceptable, but may hamper the acceptance of  some defenses. Since their behavior is similar to that of defeater in defeasible logic, we call them \textit{defeaters of defenses (DoD)}.

 \begin{definition}[Defeaters of defenses]   \label{def-defeater}
Let $\mathcal{F} = (\AR, \att)$ be an argument graph. 
For $\alpha, \beta\in \AR$, % when $\attackers{\beta}\supsetneq \{\beta\}$, $[\alpha,\beta]$ is an undercutting defense iff
\begin{itemize}
\item $(\alpha,\beta)$ is a DoD, iff $\{\alpha, \beta\}$ is not conflict-free, and
$\exists \gamma\in \AR\setminus \{\alpha,\beta\}$ such that $\attack{\alpha}{\gamma}$ and  $\attack{\gamma}{\beta}$. 
\item $(\o, \beta)$ is a DoD, iff $\beta$ is self-attacked or attacked by a self-attacked argument.
\end{itemize}

 The set of DoDs of $\mathcal{F}$ is denoted as $\und{\mathcal{F}} $.% For convenience, when without causing confusion, an undercutting defense is also called a defense.
\end{definition}

In this definition, note that $(\alpha,\beta)$ and $(\o,\beta)$ are not accepted as a defense, and may be used to hamper the acceptance of some defenses. This does not mean that the arguments $\alpha$ and $\beta$ in the corresponding argument graph can not be accepted, since they may in some defenses at the same time.  See the following example. 

\begin{example}\label{ex-defeater}
In $\mathcal{F}_{9}$, $\tuple{a, c}$, while $(\o, b)$,  $(\o, c)$ and  $(b,d)$ are DoDs. $c$ is both in the defense $\tuple{a, c}$ and in the DoD $(\o, c)$. When $\tuple{a, c}$ is accepted, $c$ is accepted. 

 
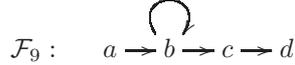
\begin{picture}(206,25)
 \put(0,2){\xymatrix@C=0.4cm@R=0.2cm{
 \mathcal{F}_{9}:  & a\ar[r] & b\ar[r]\ar@(ul, ur)& c\ar[r] & d  }}
\end{picture}
\end{example}

Note also that  in Definition \ref{def-defeater} when $\beta$ is  attacked by a self-attacked argument, it is a DoD. Consider $\mathcal{F}_8$ in Example \ref{ex-undercutter}. $\tuple{a_{11}, a_{13}}$ is a defense. If $(\o, a_{15})$ is not a DoD, then there is no DoD to prevent the acceptance of $\tuple{a_{11}, a_{13}}$.

Let $\mathrm{arg}(\und{\mathcal{F}} ) = \{\alpha, \beta\mid (\alpha,\beta) \in \und{\mathcal{F}} \} \cup  \{\beta\mid (\o,\beta) \in \und{\mathcal{F}} \}$ be the set of arguments involved in $\und{\mathcal{F}} $. Let $\attackees{\mathrm{def}(\nmd{\mathcal{F}} )}$ be the set of arguments attacked by $\mathrm{def}(\nmd{\mathcal{F}} )$. We have the following proposition.

\begin{proposition}\label{pro-com-1}
Let $\mathcal{F} = (\AR, \att)$ be an argument graph. It holds that $\AR = \mathrm{arg}(\und{\mathcal{F}} )\cup\mathrm{def}(\nmd{\mathcal{F}} )\cup\attackees{\mathrm{def}(\nmd{\mathcal{F}} )}$. 
\end{proposition}

This proposition states that arguments in $\mathcal{F}$ are equivalent to the union of the arguments in defenses, arguments in defeaters of defenses, and the arguments attacked by the arguments in defenses. 

Let $\dgn{\mathcal{F}} = \nmd{\mathcal{F}}  \cup \und{\mathcal{F}} $ be the set of defenses and their defeaters. The attack relation between the elements of $\dgn{\mathcal{F}} $ can be identified according to the attack relation between the arguments involved. For convenience,  we also write $[x, \beta]$ to denote a defense $\tuple{x, \beta}$ or a defeater of defenses $(x, \beta)$ where $x\in\AR\cup\{\o\}$ and $\beta\in \AR$. Formally we have the following definition.

\begin{definition}[Attacks between defenses and their defeaters] \label{def-avdadg}
For all  $[x, \alpha], [y, \beta] \in \dgn{\mathcal{F}} $ where $x,y\in \AR\cup \{\o\}$ and $\alpha, \beta\in \AR$, we say that  $[x, \alpha]$ attacks $[y, \beta]$, denoted as $[x, \alpha] \rightarrow^\mathrm{d} [y, \beta]$ iff $\attack{x}{y}$, $\attack{x}{\beta}$ , $\attack{\alpha}{y}$, or  $\attack{\alpha}{\beta}$.
\end{definition}

The set of attacks between defenses and their defeaters is denoted as $\rightarrow^\mathrm{d}$. Given $D\subseteq \dgn{\mathcal{F}}$ and $X\in \dgn{\mathcal{F}}$, we use $\attackd{D}{X}$ to denote that $\exists Y\in D$ such that $\attackd{Y}{X}$. 

Since the status of a defense is determined by that of other defenses and affected by defeaters of defenses through the attacks between them, to evaluate the status of normal defenses, one possible way is to use \textit{defense graph}, which is defined as follows. %Note that each undercutting defense is self-attacking, it can not be accepted. 

\begin{definition}[Defense graph]
Let $\mathcal{F} = (\AR, \att)$ be an argument graph. Let $\dgn{\mathcal{F}} = \nmd{\mathcal{F}}  \cup \und{\mathcal{F}} $. A defense graph w.r.t. $\mathcal{F}$, denoted $\dg{\mathcal{F}}$, is defined as follows. 
\begin{equation}
\dg{\mathcal{F}}= (\dgn{\mathcal{F}} ,  \rightarrow^\mathrm{d})
\end{equation}
\end{definition}

%A defense graph is composed of a set of defenses and a set of attacks between them. 
A defense graph can be viewed as a kind of  meta-argumentation \cite{DBLP:journals/sLogica/BoellaGTV09}. %For convenience, we use $[x,y]\in \dgn{\mathcal{F}} $ to represent a defense $\tuple{x,y}$ or a defeater of defenses $(x,y)$.

\begin{example}
The defense graph of $\mathcal{F}_6$ is as follows. 

  
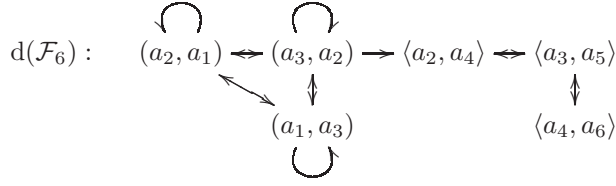
\begin{picture}(206,64)
 \put(0,40){\xymatrix@C=0.4cm@R=0.4cm{  
 \dg{\mathcal{F}_6}:  &(a_2, a_1)\ar[r]\ar[rd]\ar@(ur,ul) &(a_3, a_2)\ar[r]\ar[d]\ar[l]\ar@(ul,ur)&\tuple{a_2, a_4}\ar[r]&\tuple{a_3, a_5}\ar[l]\ar[d]\\
&& (a_1, a_3)\ar[u]\ar[ul]\ar@(dl,dr)&&\tuple{a_4, a_6}\ar[u]
  }}
      \end{picture}
\end{example}

\section{Defense semantics}
In a defense graph $\dg{\mathcal{F}} = (\dgn{\mathcal{F}} ,  \rightarrow^\mathrm{d})$, nodes are defenses and/or defeaters of defenses, rather than arguments in the corresponding argument graph $\mathcal{F}$. So,  when applying classical semantics to $\dg{\mathcal{F}}$, we get a set of extensions, each of which is a set of defenses. By slightly modifying the definition for classical semantics, defense semantics can be defined as follows.  

\begin{definition}[Defense semantics]
Defense semantics is a function $\mathrm{\Sigma}$ mapping each defense graph to a set of extensions of defenses.  Given a defense graph $\dg{\mathcal{F}}= (\dgn{\mathcal{F}} ,  \rightarrow^\mathrm{d})$ where $\dgn{\mathcal{F}} = \nmd{\mathcal{F}}  \cup \und{\mathcal{F}} $, let $D\subseteq \nmd{\mathcal{F}} $. We have:
\begin{itemize}
\item $D$ is conflict-free iff $\nexists X, Y\in D$ such that $\attackd{X}{Y}$. 
\item $X\in \nmd{\mathcal{F}} $ is defended by $D$ iff for all $Y\in \dgn{\mathcal{F}} $, if $\attackd{Y}{X}$, then $\exists Z\in D$ such that $\attackd{Z}{Y}$.
\item $D$ is admissible iff $D$ is conflict-free and each member in $D$ is defended by $D$.
\item $D$ is a complete extension of defenses iff $D$ is admissible, and each member in $\nmd{\mathcal{F}} $ that is defended by $D$ is in $D$.
\item $D$ is the {grounded extension} of defenses iff $D$ is the minimal (w.r.t. set-inclusion) complete extension of defenses.
\item $D$ is a {preferred extension} of defenses iff $D$ is a maximal (w.r.t. set-inclusion) complete extension of defenses.
\item  $D$ is a {stable extension} of defenses iff  $D$ is conflict-free, and $\forall X\in \dgn{\mathcal{F}}\setminus D$,  $\attackd{D}{X}$.
\end{itemize}
\end{definition}

The set of  complete, grounded, preferred, and stable extensions of defenses of $\dg{\mathcal{F}}$ is denoted as $\mathrm{CO}(\dg{\mathcal{F}})$, $\mathrm{GR}(\dg{\mathcal{F}})$, $\mathrm{PR}(\dg{\mathcal{F}})$ and $\mathrm{ST}(\dg{\mathcal{F}})$ respectively. 

Note that the notion of defense semantics is similar to that of classical semantics. The only difference is that in a defense graph, we differentiate two kinds of nodes: defenses and defeaters of defenses. The former can be included in extensions, while the latter are only used to prevent the acceptance of some defenses. 

Now, let us consider some properties of the defense semantics of an argument graph.

The first property is about the relation between defense semantics and classical semantics. Let $D\in \Sigma(\dg{\mathcal{F}})$ be a $\Sigma$-extension of $\dg{\mathcal{F}}$.  Now the question is whether the set of defenders and defendees in $D$ is a $\sigma$-extension of $\mathcal{F}$.  In order to verify this property, technically, we first present the follow lemma. The lemma states that $\forall\tuple{\alpha, \beta}\in D$, if $\alpha$ is attacked by an argument $\gamma\in \AR$, then $\exists \eta\in  \mathrm{def}(D)$ such that $\eta$ attacks $\gamma$.

\begin{lem}\label{lem-11}
For all $D\in \Sigma(\dg{\mathcal{F}})$, for all $\tuple{x, y}\in D$,  for all $\gamma\in  \AR$, if $\attack{\gamma}{x}$ or $\attack{\gamma}{y}$, then $\attack{\mathrm{def}(D)}{\gamma}$. 
\end{lem}

\begin{example} \label{ex-constr}
Consider  $\dg{\mathcal{F}_{1}}$ below. Under complete semantics,  $\mathrm{CO}(\dg{\mathcal{F}_{1}})= \{D_1$, $D_2$, $D_3\}$ where $D_1=\{ \}$, $D_2=\{\tuple{a, c_2}$, $\tuple{c_2, c_3}$, $\tuple{c_3, a}\}$, $D_3=\{\tuple{b, c_4}, \tuple{c_1, b}, \tuple{c_4, c_1}\}$.  
%$ \mathrm{co}(\dg{\mathcal{G}_{3}})= \{D_4\}$ where $D_4=\{ \tuple{\o, a}, \tuple{a, c}\}$. 
Take $D_2$ and $\tuple{a, c_2}$ in $D_2$
% and $\tuple{a, c}$ in $D_4$ 
as an example.  $ \mathrm{def}(D_2) = \{a, c_2, c_3\}$. For $a$ being attacked by $c_4$, and $c_2$ being attacked by $c_1$, it holds that $\attack{\mathrm{def}(D_2)}{c_4}$ and $\attack{\mathrm{def}(D_2)}{c_1}$.

  
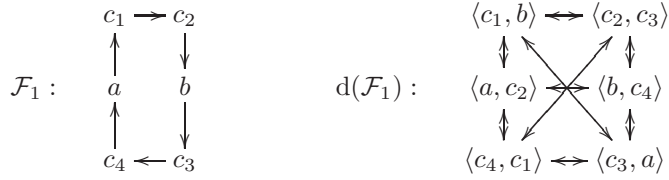
\begin{picture}(206,74)
 \put(0,60){\xymatrix@C=0.4cm@R=0.40cm{
&c_1\ar[r]&c_2\ar[d]&&&&\tuple{c_1, b}\ar[r]\ar[d]\ar[rdd]&\tuple{c_2, c_3}\ar[l]\ar[d]\ar[ldd] \\
\mathcal{F}_1: & a\ar[u]&b\ar[d]&&& \dg{\mathcal{F}_1}: &  \tuple{a, c_2}\ar[u] \ar[r]\ar[d] &\tuple{b, c_4}\ar[u]\ar[d] \ar[l] &   \\
&c_4\ar[u]&c_3\ar[l]&&&& \tuple{c_4, c_1}\ar[u]\ar[r]\ar[ruu]&\tuple{c_3, a}\ar[u]\ar[l]\ar[luu] &
 }}
 
 \end{picture}
\end{example}

Based on Lemma \ref{lem-11}, under complete semantics, we have the following theorem. 

\begin{theorem} \label{prop-ad}
For all $D \in \mathrm{CO}(\dg{\mathcal{F}})$,  $\mathrm{def}(D) \in \mathrm{co}(\mathcal{F})$.
\end{theorem}

Theorem \ref{prop-ad} makes clear that for each complete defense extension $D$ of a defense graph, there exists a complete argument extension $E$ of the corresponding argument graph such that $E$ is equal to $\mathrm{def}(D)$. On the other hand, the following theorem says that for each complete argument extension $E$ of an argument graph, there exists a complete defense extension $D$ of the corresponding defense graph such that  $D = \mathrm{d}(E)$ where $ \mathrm{d}(E) = \{\tuple{x,y}\in \nmd{\mathcal{F}} \mid x\in E\cup\{\o\} ,  y\in E\}$.

\begin{theorem} \label{prop-adco}
For all $E \in \mathrm{co}(\mathcal{F})$, $\mathrm{d}(E) \in \mathrm{CO}(\dg{\mathcal{F}})$.
\end{theorem}

The relation between argument extensions and defense extensions under other semantics is presented in the following corollaries. 

\begin{corollary}\label{cor-com}
$\forall \Sigma\in\{\mathrm{GR}, \mathrm{PR}, \mathrm{ST}\}$, it holds that $\forall D \in \Sigma(\dg{\mathcal{F}})$,  $\mathrm{def}(D) \in \sigma(\mathcal{F})$.
\end{corollary}

Proofs for Lemma \ref{lem-11},  Theorem  \ref{prop-ad}, \ref{prop-adco}
 and Corollary \ref{cor-com} are presented in the Appendix. In the following theorems and corollaries, when we say $\Sigma\in\{\mathrm{CO}, \mathrm{GR}, \mathrm{PR}, \mathrm{ST}\}$, $\sigma$ is referred to $\mathrm{co}$, $\mathrm{gr}$, $\mathrm{pr}$  and  $\mathrm{st}$, correspondingly. Meanwhile, when we say $\sigma\in\{\mathrm{co}, \mathrm{gr}, \mathrm{pr}, \mathrm{st}\}$, $\Sigma$ is referred to $\mathrm{CO}$, $\mathrm{GR}$, $\mathrm{PR}$  and  $\mathrm{ST}$, correspondingly.

\begin{corollary}\label{cor-soun}
 $\forall \Sigma\in\{\mathrm{GR}, \mathrm{PR}, \mathrm{ST}\}$  it holds that $\forall E \in \sigma(\mathcal{F})$, $\mathrm{d}(E) \in \Sigma(\dg{\mathcal{F}})$.
\end{corollary}

\begin{proof}
Under grounded semantics, we need to verify that $\mathrm{d}(E)$ is minimal (w.r.t. set-inclusion). Assume the contrary. Then $\exists D^\prime\subsetneq \mathrm{d}(E)$ such that $D^\prime$ is a grounded extension. According to theorem  \ref{prop-ad}, $\mathrm{def}(D^\prime)$ is a complete extension. It follows that $\mathrm{def}(D^\prime) \subsetneq \mathrm{def}(\mathrm{d}(E))=E$. It turns out that $E$ is not a grounded extension. Contradiction. 

Under preferred semantic, it is easy to verify that $\mathrm{d}(E)$ is maximal (w.r.t. set-inclusion).

Under stable semantics,  we need to prove that for all $[x, \alpha]\in \dgn{\mathcal{F}} \setminus \mathrm{d}(E)$: $\attack{ \mathrm{d}(E)}{[x, \alpha]}$. Assume the contrary. Then,  $\exists [x, \alpha]\in \dgn{\mathcal{F}} \setminus \mathrm{d}(E)$ such that ${ \mathrm{d}(E)}$ does not attack ${[x, \alpha]}$. So, $E$ does not attack $x$ and $\alpha$. Since $E$ is stable, it holds that $\{x, \alpha\}\setminus \{\o\} \subseteq E$. So, $[x, \alpha]\in E$. Contradiction. 
\end{proof}

Theorems \ref{prop-ad} and \ref{prop-adco} and Corollaries \ref{cor-com} and \ref{cor-soun} describe the relation between argument extensions and defense extensions under various semantics. This relation can be further described by two equations in the following two corollaries. 
First, by overloading the notation, let $\mathrm{d}(\sigma(\mathcal{F})) = \{\mathrm{d}(E) \mid E\in \sigma(\mathcal{F})\}$, where $\sigma\in \{\mathrm{co}, \mathrm{gr}, \mathrm{pr},\mathrm{st}\}$. 

\begin{corollary}\label{cor-1}
For all $\sigma\in\{\mathrm{co}, \mathrm{pr}, \mathrm{gr}, \mathrm{st}\}$, it holds that $\mathrm{d}(\sigma(\mathcal{F})) = \Sigma(\dg{\mathcal{F}})$.
\end{corollary}

\begin{proof}
For all $\mathrm{d}(E)\in \mathrm{d}(\sigma(\mathcal{F}))$, according to Theorem \ref{prop-adco} and Corollary  \ref{cor-soun}, $\mathrm{d}(E)\in  \Sigma(\dg{\mathcal{F}})$. For all $D\in \Sigma(\dg{\mathcal{F}})$, according to Theorem \ref{prop-ad} and Corollary \ref{cor-com}, $\mathrm{def}(D)\in \sigma(\mathcal{F})$.  Since $\mathrm{d}(\mathrm{def}(D)) =  \{\tuple{\beta,\alpha}\in \nmd{\mathcal{F}} \mid \alpha,\beta\in \mathrm{def}(D)\} \cup \{\tuple{\o,\alpha}\in \nmd{\mathcal{F}} \mid \alpha\in \mathrm{def}(D)\} = D$, it holds that $D\in \mathrm{d}(\sigma(\mathcal{F}))$.  
\end{proof}

\begin{example}\label{ex-cr1}
Consider $\mathcal{F}_{10}$ and $\dg{\mathcal{F}_{10}}$ below. Under complete semantics, we have:
\begin{itemize}
\item $\mathrm{co}(\mathcal{F}_{10}) = \{E_1, E_2\}$, where $E_1=\{ \}$, $E_2 =  \{b\}$;
\item $\mathrm{d}(\mathrm{co}(\mathcal{F}_{10})) = \{\mathrm{d}(E_1), \mathrm{d}(E_2)\}$, where $\mathrm{d}(E_1) = \{ \}$, $\mathrm{d}(E_2) = \{\tuple{b,b}\}$;
\item $ \mathrm{CO}(\dg{\mathcal{F}_{10}}) = \{D_1, D_2\}$, where $D_1 = \{ \}$, $D_2 = \{\tuple{b,b}\}$.
\end{itemize}

So, it holds that $\mathrm{d}(\mathrm{co}(\mathcal{F}_{10})) = \mathrm{CO}(\dg{\mathcal{F}_{10}})$. 

  
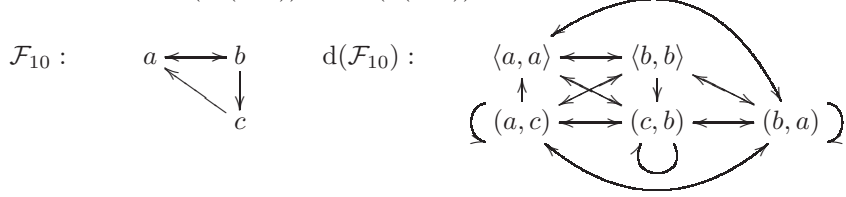
\begin{picture}(206,70)
\put(0,50){\xymatrix@C=0.8cm@R=0.3cm{
\mathcal{F}_{10}: &a\ar[r]&b\ar[l]\ar[d]&\dg{\mathcal{F}_{10}}: &\tuple{a, a}\ar[r] \ar@/^1.2cm/[rrd] \ar[rd]&\tuple{b, b} \ar[rd] \ar[l]\ar[ld]\ar[d]\\
&&c\ar[lu]&& (a, c)\ar@/_0.9cm/[rr]\ar[ru]\ar@(lu,ld)\ar[r]\ar[u]  &(c, b)\ar@(dr,dl)\ar[l] \ar[r]\ar[lu] &(b, a)\ar@/_1.2cm/[llu] \ar@(ru,rd)\ar[l] \ar[lu]\ar@/^0.9cm/[ll]
 }}
\end{picture}
\end{example}

Second, by overloading the notation, let $\mathrm{def}(\Sigma(\dg{\mathcal{F}})) = \{\mathrm{def}(D) \mid D\in \Sigma(\dg{\mathcal{F}})\}$, where $\Sigma\in \{\mathrm{CO}, \mathrm{GR}, \mathrm{PR},\mathrm{ST}\}$.

\begin{corollary}\label{cor-2}
For all $\Sigma\in \{\mathrm{CO}, \mathrm{GR}, \mathrm{PR},\mathrm{ST}\}$,  it holds that $\sigma(\mathcal{F}) = \mathrm{def}(\Sigma(\dg{\mathcal{F}}))$.
\end{corollary}

\begin{proof}
For all $E \in \sigma(\mathcal{F})$, according to Theorem \ref{prop-adco} and Corollary  \ref{cor-soun}, $\mathrm{d}(E)\in  \Sigma(\dg{\mathcal{F}})$. Since $\mathrm{def}(\mathrm{d}(E))= E$, $E\in  \mathrm{def}(\Sigma(\dg{\mathcal{F}}))$. For all $\mathrm{def}(D) \in \mathrm{def}(\Sigma(\dg{\mathcal{F}}))$, since $D\in  \Sigma(\dg{\mathcal{F}})$, according to Theorem \ref{prop-ad}  and Corollary \ref{cor-com}, $\mathrm{def}(D)\in \mathrm{co}(\mathcal{F})$.   
\end{proof}

\begin{example}
Under compete semantics, continue Example \ref{ex-cr1}, $\mathrm{def}(\mathrm{CO}(\dg{\mathcal{F}_{10}})) = \{\mathrm{def}(D_1)$, $\mathrm{def}(D_2)\}$, where $\mathrm{def}(D_1) = \{ \}$, $\mathrm{def}(D_2) = \{b\}$. It holds that $\mathrm{co}(\mathcal{F}_{10}) = \mathrm{def}(\mathrm{CO}(\dg{\mathcal{F}_{10}}))$.
\end{example}

The second property formulated in Theorems \ref{th-s-d}, \ref{th-s-d2} is about the equivalence of argument graphs under defense semantics, called \textit{defense equivalence of argument graphs}. 

\begin{definition}[Defense equivalence of AFs]
Let  $\mathcal{F}$  and $\mathcal{G}$ be two argument graphs.
 $\mathcal{F}$  and $\mathcal{G}$ are of defense equivalence w.r.t. a semantics $\Sigma$, denoted as $\mathcal{F} \equiv_\mathrm{d}^\Sigma \mathcal{G}$, iff $\Sigma(\dg{\mathcal{F}}) = \Sigma(\dg{\mathcal{G}})$.
\end{definition}

Concerning the relation between defense equivalence and standard equivalence of argument graphs, we have the following theorem. 

\begin{theorem}\label{th-s-d}
Let  $\mathcal{F}$  and $\mathcal{G}$ be two argument graphs, and $\Sigma\in\{\mathrm{CO}, \mathrm{GR}, \mathrm{PR},\mathrm{ST}\}$ be a semantics. If $\mathcal{F} \equiv_\mathrm{d}^\Sigma \mathcal{G}$, then $\mathcal{F} \equiv^\sigma \mathcal{G}$.
\end{theorem}

\begin{proof}
If $\mathcal{F} \equiv_\mathrm{d}^\Sigma \mathcal{G}$, then $\Sigma(\dg{\mathcal{F}}) =\Sigma(\dg{\mathcal{G}})$. According to Corollary \ref{cor-2}, it follows that $\sigma(\mathcal{F}) =  \mathrm{def}(\Sigma(\dg{\mathcal{F}})) =  \mathrm{def}(\Sigma(\dg{\mathcal{G}})) = \sigma(\mathcal{G})$. Since  $\sigma(\mathcal{F}) =  \sigma(\mathcal{G})$, $\mathcal{F} \equiv^\sigma  \mathcal{G}$.  
\end{proof}

Note that $\mathcal{F} \equiv^\sigma \mathcal{G}$ does not imply $\mathcal{F} \equiv_\mathrm{d}^\Sigma \mathcal{G}$ in general. Consider the following example under complete semantics.

\begin{example} \label{ex-sdd}
Since $\mathrm{co}(\mathcal{F}_3) =  \mathrm{co}(\mathcal{F}_4) = \{\{a, c\}\}$, it holds that $\mathcal{F}_3 \equiv^\mathrm{co}  \mathcal{F}_4$. Since $\mathrm{CO}(\dg{\mathcal{F}_3}) = \{\{\tuple{\o, a}\}, \tuple{a,c} \}$ and $\mathrm{CO}(\dg{\mathcal{F}_4}) = \{\{\tuple{\o, a}\}, \tuple{\o,c} \}$,  $\mathrm{CO}(\dg{\mathcal{F}_3})\neq \mathrm{CO}(\dg{\mathcal{F}_4})$. So, it is not the case that $\mathcal{F}_3 \equiv_\mathrm{d}^\mathrm{CO} \mathcal{F}_4$.

 \begin{picture}(206,40)
 \put(0,27){\xymatrix@C=0.3cm@R=0.2cm{
  \mathcal{F}_3: \hspace{0.3cm} a\ar[r]&b\ar[r]&c && \dg{\mathcal{F}_3}: &\tuple{\o,a}& \tuple{a,c} \\
\mathcal{F}_4:  \hspace{0.3cm}  a\ar[r]&b&c && \dg{\mathcal{F}_4}: &\tuple{\o,a}& \tuple{\o,c}
}}
 \end{picture}
\end{example}

About the relation between defense equivalence and strong equivalence of argument graphs,  under complete semantics, we have the following lemma and theorem. 

\begin{lem} \label{lem-1}
It holds that $\mathrm{CO}(\dg{\mathcal{F}}) =\mathrm{CO}(\dg{\mathcal{F}^\mathrm{ck}})$.
\end{lem}

\begin{proof}
According to Corollary \ref{cor-1}, $\mathrm{d}(\mathrm{co}(\mathcal{F})) = \mathrm{CO}(\dg{\mathcal{F}})$, $\mathrm{d}(\mathrm{co}(\mathcal{F}^\mathrm{ck})) = \mathrm{CO}(\dg{\mathcal{F}^\mathrm{ck}})$. Since $\mathrm{co}(\mathcal{F}) = \mathrm{co}(\mathcal{F}^\mathrm{ck})$, $\mathrm{CO}(\dg{\mathcal{F}}) = \mathrm{d}(\mathrm{co}(\mathcal{F})) = \mathrm{d}(\mathrm{co}(\mathcal{F}^\mathrm{ck})) = \mathrm{CO}(\dg{\mathcal{F}^\mathrm{ck}})$. 
\end{proof}

\begin{theorem}\label{th-s-d2}
Let  $\mathcal{F}$  and $\mathcal{G}$ be two argument graphs. If $\mathcal{F} \equiv_s^\mathrm{co} \mathcal{G}$, then $\mathcal{F} \equiv_\mathrm{d}^\mathrm{CO}  \mathcal{G}$.
\end{theorem}

\begin{proof}
 If $\mathcal{F} \equiv_s^\mathrm{co} \mathcal{G}$, then $\mathcal{F}^\mathrm{ck} = \mathcal{G}^\mathrm{ck}$. So, $\mathrm{CO}(\dg{\mathcal{F}^\mathrm{ck}}) = \mathrm{CO}(\dg{\mathcal{G}^\mathrm{ck}})$. According to Lemma  \ref{lem-1}, $\mathrm{CO}(\dg{\mathcal{F}}) = \mathrm{CO}(\dg{\mathcal{F}^\mathrm{ck}})$, $\mathrm{CO}(\dg{\mathcal{G}}) = \mathrm{CO}(\dg{\mathcal{G}^\mathrm{ck}})$.  So, we have $\mathrm{CO}(\dg{\mathcal{F}}) = \mathrm{CO}(\dg{\mathcal{G}})$, i.e., $\mathcal{F} \equiv_\mathrm{d}^\mathrm{CO}  \mathcal{G}$. 
\end{proof}

Note that $\mathcal{F} \equiv_\mathrm{d}^\mathrm{CO}  \mathcal{G}$ does not imply $\mathcal{F} \equiv_s^\mathrm{co} \mathcal{G}$ in general. Consider the following example.

\begin{example}
Since $\mathrm{CO}(\dg{\mathcal{F}_3}) = \mathrm{CO}(\dg{\mathcal{F}_{11}}) =  \{\{\tuple{\o, a}\}$, $\tuple{a,c} \}$,  $\mathcal{F}_3 \equiv_\mathrm{d}^\mathrm{CO}  \mathcal{F}_{11}$. However, since $\mathcal{F}_3^\mathrm{ck} \neq \mathcal{F}_{11}^\mathrm{ck}$, $\mathcal{F}_3 \not\equiv_s^\mathrm{co}  \mathcal{F}_{11}$.

 \begin{picture}(206,30)
 \put(0,19){\xymatrix@C=0.3cm@R=0.2cm{
\mathcal{F}_{11}:  & a\ar[r]\ar[rd]&b\ar[r]&c && \dg{\mathcal{F}_{11}}: &\tuple{\o,a}& \tuple{a,c} \\
&&d\ar[ru]&  
}}
 \end{picture}
\end{example}

\section{Encoding reasons for accepting arguments}
Defense semantics can be used to encode reasons for accepting arguments. Consider the following example.
\begin{example}\label{ex-eraa}
$\mathrm{CO}(\dg{\mathcal{F}_{12}}) = \{D_1, D_2\}$, where $D_1 = \{\tuple{b,b}$, $\tuple{b,d}$, $\tuple{g,d}, \tuple{e,g}, \tuple{\o,e}\}$, $D_2 = \{\tuple{a,a}$, $\tuple{g,d}$, $\tuple{e,g}$, $\tuple{\o,e}\}$.  One way to capture reasons for accepting arguments is to relate each reason to an extension of defenses. For instance, concerning the reasons for accepting $d$ w.r.t. $D_1$, we differentiate the following reasons:
\begin{itemize}
\item Direct reason: accepting \{b, g\} is a direct reason for accepting $d$. This reason can be identified in terms of defenses $\tuple{b,d}$ and $\tuple{g,d}$ in $D_1$. %\item Indirect reason: accepting  \{c, h\} is an  indirect reason of distance 2 for accepting $g$; accepting $\{a\}$ is an  indirect reason  of distance 3 for accepting $g$.
\item Root reason: accepting $\{e, b\}$ is a root reason for accepting $d$, in the sense that each element of a root reason is either an initial argument, or an argument without further defenders except itself. This reason can be identified by means of viewing each defense as a binary relation,  and allowing this relation to be transitive. Given $\tuple{e,g}$ and $\tuple{g,d}$ in $D_1$, we have $\tuple{e,d}$. Since $e$ is an initial argument, it is an element of the root reason. Given $\tuple{b,d}$ in $D_1$, since $b$'s defender is  $b$ itself, $b$ is an element of the root reason. 
%\item Minimal reason: accepting $\{a\}$ or $\{c\}$ or $\{e\}$ or $\{h\}$ or $\{j\}$ is a  minimal reason for accepting $g$. 
\end{itemize}

 
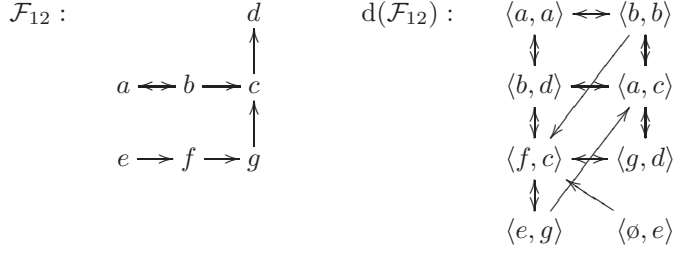
\begin{picture}(206,100) 
 \put(0,90){\xymatrix@C=0.45cm@R=0.4cm{
  \mathcal{F}_{12}:  &  & &d && \dg{\mathcal{F}_{12}}:  & \tuple{a,a}\ar[r]\ar[d] & \tuple{b,b} \ar[l]\ar[d]\ar[ddl]\\
&a\ar[r]&b\ar[l]\ar[r] &c \ar[u] &&& \tuple{b,d}\ar[u] \ar[r] \ar[d]& \tuple{a,c}\ar[u]\ar[l] \ar[d] \\
&e\ar[r] &f \ar[r] & g\ar[u]& & &\tuple{f,c}\ar[r]\ar[d]\ar[u]& \tuple{g,d}\ar[u] \ar[l] \\
&&&&&& \tuple{e,g}\ar[u]\ar[uur]& \tuple{\o,e}\ar[lu]}} 
    \end{picture}
\end{example}

%When allowing defense relation to be transitive, 
%
%\begin{definition} [General defense]
%Let $\dgn{\mathcal{F}} $ be the set of defenses of $\mathcal{F}$. General defense relation is recursively defined as follow:
%\begin{itemize}
%\item for all $\tuple{\alpha, \beta}\in \dgn{\mathcal{F}} $,  $\tuple{\alpha, \beta}$ is a general defense;
%\item if  $\tuple{\alpha, \beta}$ and $\tuple{\beta, \gamma}$ are general defenses, then $\tuple{\alpha, \gamma}$ is a general defense;
%\item no other relation is a general defense. 
%\end{itemize}
%\end{definition}
%
%General defense relation is transitive. 
The informal notions in Example \ref{ex-eraa} are formulated as follows. 

\begin{definition}[Direct reasons for accepting arguments]\label{def-draa}
Let $\mathcal{F} = (\AR, \rightarrow)$ be an argument graph. Direct reasons for accepting arguments in $\mathcal{F}$ under a semantics $\Sigma$ is a function, denoted $\mathrm{dr}_\Sigma^{\mathcal{F}}$, mapping from $\mathcal{F}$ to sets of arguments, such that for all $\alpha\in \AR$, 
\begin{equation}
\mathrm{dr}_\Sigma^{\mathcal{F}}(\alpha) = \{\mathcal{DR}(\alpha, D) \mid D\in \Sigma(\dg{\mathcal{F}})\} 
\end{equation}
where 
$\mathcal{DR}(\alpha, D) = \{\beta\mid \tuple{\beta, \alpha}\in D\}$, if $\alpha$ is not an initial argument; otherwise, $\mathcal{DR}(\alpha, D) = \{\o\}.$
\end{definition}

\begin{example}
Continue Example \ref{ex-eraa}. According to Definition \ref{def-draa}, $\mathrm{dr}^{\mathcal{F}_{12}}_\mathrm{CO}(d) = \{R_1, R_2\}$, where $R_1 = \{b,g\}$, $R_2 = \{g\}$. $\mathrm{dr}^{\mathcal{F}_{12}}_\mathrm{CO}(f) = \{R_3, R_4\}$, where $R_3 = R_4 = \{ \}$. 
\end{example}

For all $D\in \Sigma(\dg{\mathcal{F}})$, we view $D$ as a transitive relation, and let $D^+$ be the transitive closure of $D$.

\begin{definition}[Root reasons for accepting arguments] \label{def-rraa}
Let $\mathcal{F} = (\AR, \rightarrow)$ be an argument graph.  Root reasons for accepting arguments in $\mathcal{F}$ under a semantics $\Sigma$ is a function, denoted $\mathrm{rr}^{\mathcal{F}}_\Sigma$, mapping from $\mathcal{F}$ to sets of arguments, such that for all $\alpha\in \AR$, 
\begin{equation}
\mathrm{rr}^{\mathcal{F}}_\Sigma(\alpha) = \{\mathcal{RR}(\alpha, D) \mid D\in \Sigma(\dg{\mathcal{F}})\}
\end{equation}
where $\mathcal{RR}(\alpha, D)  =\{\beta\in\AR\mid \tuple{\beta,\beta}\in D^+,  \beta = \alpha\} \cup  \{\beta\in\AR \mid  (\tuple{\beta, \alpha}\in D^+) ,  (\tuple{\beta, \beta}\in D^+\vee \attackers{\beta}=\emptyset) \}$,
if $\alpha$ is not  initial; otherwise, $\mathcal{RR}(\alpha, D) = \{\o\}$.
\end{definition}

According Definition \ref{def-rraa}, we say that a set of arguments $\mathcal{RR}(\alpha, D)$ is a root reason of an argument $\alpha$ iff for all $\beta\in \mathcal{RR}(\alpha, D)$, $\beta$ is either equal to $\alpha$ when $\alpha$ (partially) defends itself directly or indirectly through a transitive relation of defenses in $D$, or an initial argument, or an argument that can (partially) defend itself directly or indirectly.

\begin{example}
Continue Example \ref{ex-eraa}.  $D_1^+ = D_1 \cup \{\tuple{e,d}$, $\tuple{\o,g}$, $\tuple{\o,d}\}$; $D_2^+ = D_2 \cup \{\tuple{e,d}, \tuple{\o,g}, \tuple{\o,d}\}$. According to Definition \ref{def-rraa}, $\mathrm{rr}^{\mathcal{F}_{12}}_\mathrm{CO}(d) = \{R_1, R_2\}$, where $R_1 = \{b,e\}$, $R_2 = \{e\}$. $\mathrm{rr}^{\mathcal{F}_{12}}_\mathrm{CO}(f) = \{R_3, R_4\}$, where $R_3 = R_4 = \{ \}$. 
\end{example}

Motivated by the first example in Section 1 (regarding $\mathcal{F}_1$), based on the notion of root reasons, we propose as follows a notion of root equivalence of AFs. 

\begin{definition}[Root equivalence of AFs]
Let $\mathcal{F} = (\AR_1, \att_1)$ and $\mathcal{H} = (\AR_2, \att_2)$ be two argument graphs. For all $B\subseteq \AR_1\cap \AR_2$, if $B \neq \emptyset$, we say that $\mathcal{F}$ and $\mathcal{H}$ are equivalent w.r.t. the root reasons for accepting $B$ under semantics $\Sigma$, denoted $\mathcal{F}|B \equiv^\Sigma_{rr} \mathcal{H}|B$, iff for all $\alpha\in B$, $\mathrm{rr}_\Sigma^{\mathcal{F}}(\alpha) = \mathrm{rr}_\Sigma^{\mathcal{H}}(\alpha) $. 
\end{definition}

When $B = \AR_1 = \AR_2$, we write $\mathcal{F} \equiv^\Sigma_{rr} \mathcal{H}$ for $\mathcal{F}|B \equiv^\Sigma_{rr} \mathcal{H}|B$.

\begin{example}
Consider $\mathcal{F}_1$ and $\mathcal{F}_2$ in Section 1 again. Under complete semantics, $\mathrm{CO}(\dg{\mathcal{F}_{1}})= \{D_1, D_2, D_3\}$ where $D_1=\{ \}$, $D_2=\{\tuple{a, c_2}$, $\tuple{c_2, c_3}$, $\tuple{c_3, a}\}$, $D_3=\{\tuple{b, c_4}, \tuple{c_1, b}, \tuple{c_4, c_1}\}$. $\mathrm{CO}(\dg{\mathcal{F}_{2}})= \{D_4, D_5, D_6\}$ where $D_4=\{ \}$, $D_5=\{\tuple{a,a}\}$, $D_6=\{\tuple{b,b}\}$. Let $B = \{a, b\}$.
 $\mathrm{rr}_\mathrm{CO}^{\mathcal{F}_1}(a) = \{\{ \}, \{a\}, \{ \}\}$, $\mathrm{rr}_\mathrm{co}^{\mathcal{F}_2}(a) =\{\{ \}, \{a\}, \{ \}\}$,
  $\mathrm{rr}_\mathrm{CO}^{\mathcal{F}_1}(b) = \{\{\}, \{ \}, \{b\}\}$, $\mathrm{rr}_\mathrm{co}^{\mathcal{F}_2}(b) = \{\{\}, \{ \}, \{b\}\}$.
 So, it holds that $\mathcal{F}_1|B =^\mathrm{CO}_\mathrm{rr} \mathcal{F}_2|B$.
\end{example}

\begin{theorem} \label{th-rr-st}
Let $\mathcal{F} = (\AR_1, \att_1)$ and $\mathcal{H} = (\AR_2, \att_2)$ be two argument graphs. If $\mathcal{F} \equiv^\mathrm{CO}_{rr} \mathcal{H}$, then $\mathcal{F} \equiv^\mathrm{co}  \mathcal{H}$.
\end{theorem}

\begin{proof}
According to Definition \ref{def-rraa}, the number of extensions of $\mathrm{co}(\mathcal{F})$ is equal to the number of  $\mathrm{rr}_\mathrm{CO}^{\mathcal{F}}(\alpha)$, where $\alpha\in \AR_1$. Since $\mathrm{rr}_\mathrm{CO}^{\mathcal{F}}(\alpha) = \mathrm{rr}_\mathrm{CO}^{\mathcal{H}}(\alpha)$, $\AR_1 = \AR_2$.
Let $\mathrm{rr}_\mathrm{CO}^{\mathcal{F}}(\alpha) = \mathrm{rr}_\mathrm{CO}^{\mathcal{H}}(\alpha) = \{R_1, \dots, R_n\}$. Let $\mathrm{co}(\mathcal{F}) = \{E_1, \dots, E_n\}$ be the set of extensions of $\mathcal{F}$, where $n \ge 1$. 
For all $\alpha\in \AR_1$, for all $R_i$, $i = 1, \dots, n$, we have $\alpha\in E_i$ iff  $R_i \neq \{\}$, in that in terms of Definition \ref{def-rraa},when $R_i \neq \{\}$, there is a reason to accept $\alpha$. 
On the other hand, let $\mathrm{co}(\mathcal{H}) =  \{S_1, \dots, S_n\}$ be the set of extensions of $\mathcal{H}$. For all $\alpha\in \AR_2 =\AR_1$, for all $R_i$, $i = 1, \dots, n$, for the same reason, we have $\alpha\in S_i$ iff  $R_i \neq \{\}$. So, it holds that $E_i = S_i$ for $i = 1, \dots, n$, and hence $\mathrm{co}(\mathcal{F}) = \mathrm{co}(\mathcal{H})$, i.e.,  $\mathcal{F} \equiv^\mathrm{co}  \mathcal{H}$. 
\end{proof}

Note that  $\mathcal{F} \equiv^\mathrm{co}  \mathcal{H}$ does not imply $\mathcal{F}\equiv^\mathrm{CO}_{rr} \mathcal{H}$ in general. This can be easily verified by considering $\mathcal{F}_3$ and $\mathcal{F}_4$ in Example \ref{ex-sdd}.

The notion of root equivalence of argument graphs can be used to capture a kind of  summarization in the graphs. Consider the following example borrowed from \cite{iomultipoles}.

\begin{example} \label{ex-sum1}
Let $\mathcal{F}_{13} = (\AR, \rightarrow)$ and $\mathcal{F}_{13} =(\AR^\prime, \rightarrow^\prime)$, illustrated below. Under complete semantics, $\mathcal{F}_{13}$ is a summarization of $\mathcal{F}_{13}$ in the sense that $\AR^\prime \subseteq \AR$, and the root reason of each argument in $\mathcal{F}_{13}$ is the same as that of each corresponding argument in $\mathcal{F}_{13}$. More specifically, it holds that $\mathrm{rr}^{\mathcal{F}_{13}}_\mathrm{CO}(e_3) = \mathrm{rr}^{\mathcal{F}_{13}}_\mathrm{CO}(e_3) = \{\{e_1, e_2\}\}$, $\mathrm{rr}^{\mathcal{F}_{13}}_\mathrm{CO}(e_2) = \mathrm{rr}^{\mathcal{F}_{13}}_\mathrm{CO}(e_2) = \{\{\o\}\}$, and $\mathrm{rr}^{\mathcal{F}_{13}}_\mathrm{CO}(e_1) = \mathrm{rr}^{\mathcal{F}_{13}}_\mathrm{CO}(e_1) = \{\{\o\}\}$.

 \begin{picture}(206,32)
 \put(0,20){\xymatrix@C=0.26cm@R=0.1cm{
\mathcal{F}_{13}:  & e_1\ar[r]&a_1\ar[r]&a_2 \ar[r]&o\ar[r]& e_3  & \mathcal{F}_{13}: & e_1\ar[r]&o\ar[r]&e_3 \\
&e_2\ar[r]&b_1\ar[r]&b_2\ar[ru]&  &&& e_2\ar[ru]
}}
 \end{picture}
\end{example}

%Formally, we have the following definition. 

\begin{definition}[Summarization of AFs]
Let $\mathcal{F} = (\AR_1, \att_1)$ and $\mathcal{H} = (\AR_2, \att_2)$ be two argument graphs. $\mathcal{F}$ is a summarization of $\mathcal{H}$ under a semantics $\sigma$ iff $\AR_1 \subset \AR_2$, and $\mathcal{F}|\AR_1 \equiv^\sigma_{rr} \mathcal{H}|\AR_1$.
\end{definition}

Now, a property of summarization of argument graphs under complete semantics is as follows. 

\begin{theorem}
Let $\mathcal{F} = (\AR_1, \att_1)$ and $\mathcal{H} = (\AR_2, \att_2)$ be two argument graphs. If $\mathcal{F}$ is a summarization of $\mathcal{H}$ under complete semantics $\mathrm{CO}$, then $\mathrm{CO}(\mathcal{F}) = \{E\cap \AR_2\mid E\in \mathrm{CO}(\mathcal{H})\}$.  
\end{theorem}

\begin{proof}
Let $\mathrm{co}(\mathcal{F}) = \{E_1, \dots, E_n\}$, $\mathrm{co}(\mathcal{H}) = \{S_1, \dots$, $S_n\}$. According to the proof of Theorem \ref{th-rr-st}, $E_1 = S_1\cap \AR_2$. Therefore, we have $\mathrm{co}(\mathcal{F}) = \{E\cap \AR_2\mid E\in \mathrm{co}(\mathcal{H})\}$.  
\end{proof}

The property looks similar to that of directionality of argumentation \cite{DBLP:journals/ai/BaroniG07}. However, they are conceptually different. Specifically,  it is said that if a semantics $\sigma$ satisfies the property of directionality iff $\forall\mathcal{F} = (\AR, \att)$, $\forall U\subseteq \AR$, if $U$ is an unattacked set, then $\sigma(\mathcal{F}\downarrow U) =  \{E\cap U\mid E\in \sigma(\mathcal{F})\}$ where $\mathcal{F}\downarrow U = (U, \att\cap(U\times U))$. So, the property of directionality is about the relation between an argument graph and its subgraph induced by an unattacked set . By contrast, the property of summarization of argument graphs is about the relation between two root equivalent argument graphs.

\section{Conclusions}
In this paper, we have proposed a defense semantics of argumentation based on a novel notion of defense graphs, and used it to encode reasons for accepting arguments. By introducing two new kinds of equivalence relation between argument graphs, i.e., defense equivalence and root equivalence, we have shown that defense semantics can be used to capture the equivalence of argument graphs from the perspective of reasons for accepting arguments. In addition, we have defined a notion of summarization of argument graphs by exploiting root equivalence.  

Under complete semantics, defense equivalence is located inbetween strong and standard equivalence. It is interesting to further investigate its position in the so-called equivalence zoo where further equivalence notions inbetween the two extremal versions are compared too \cite{BaumannB15}, and to study how defense equivalence, root equivalence and strong equivalence are related. We will present this part of work in an extended version of the present paper. 

%\textcolor{red}{We only present properties of defense semantics under complete semantics in the present paper. Further properties under some other argumentation semantics will be studied.}  
Since defense semantics explicitly represents a defense relation in extensions and can be used to encoded reasons for accepting arguments, it provides a new way to investigate topics such as summarization in argumentation \cite{iomultipoles}, dynamics of argumentation \cite{DBLP:journals/jair/CayrolSL10,DBLP:journals/ai/LiaoJK11,DBLP:journals/ai/FerrettiTGES17}, dialogical argumentation \cite{DBLP:journals/ijar/HunterT16,DBLP:conf/atal/FanT16}, etc.  Further work on these topics is promising. Meanwhile, it might be interesting to study defense semantics beyond Dung's argumentation, including ADFs \cite{DBLP:conf/kr/BrewkaW10}, bipolar frameworks  \cite{DBLP:journals/ijar/CayrolL13},  structured argumentation \cite{DBLP:journals/argcom/BesnardGHMPST14}, etc.   In particular, it would be interesting to apply defense semantics to modeling the explanation of why a conclusion can be reached. In \cite{dke:proof}, in order to increase the trust of the users for the Semantic Web applications, a system was proposed to automatically generate an explanation for every answer about why the answer has been produced. The notion of proof trace in \cite{dke:proof} for explanation is closer to the notion of support relation between arguments. So, combining the defense relation (which is based on attack relation) and support relation would be useful to model the explanation of conclusions of a structured argumentation system.  

%In addition, defense semantics for bipolar frameworks  \cite{DBLP:journals/ijar/CayrolL13} can express in addition which supporting arguments are the reason that an argument is accepted. Future work on this topic will be presented in another paper.   

\section*{Acknowledgements}
The research reported in this paper was partially supported by the National
Research Fund Luxembourg (FNR) under grant INTER/MOBILITY/14/8813732
for the project FMUAT: Formal Models for Uncertain Argumentation from Text,
and the European Union's Horizon 2020 research and innovation programme
under the Marie Sklodowska-Curie grant agreement No 690974 for the project
MIREL: MIning and REasoning with Legal texts.

%\textcolor{red}{To include a comparison to input-output equivalence introduced by \cite{DBLP:journals/ai/BaroniBCGTV14}. }

%Is there a link between the defense relation presented here and the support relation from bipolar AFs \cite{DBLP:journals/ijar/CayrolL13}.  Similarly, is there a link between the notion of summarization presented here and the summarization from \cite{iomultipoles}. Do you have any clue about a relation between root equivalence and strong equivalence?

\bibliographystyle{splncs}
%\bibliography{liao}

\section*{Appendix}
1. Proof of Lemma \ref{lem-11}
\begin{proof}
For all $\Sigma\in \{\mathrm{CO}, \mathrm{GR}, \mathrm{PR}, \mathrm{ST}\}$, it holds  that $D\in \Sigma(\dg{\mathcal{F}})$ is a complete extension. 
With respect to $\gamma$ there are the following four possible cases. Let us analyze them one by one.  
First, $\gamma$ is initial. In this case, $\tuple{x,y}$ is attacked by  $\tuple{\o,\gamma}$ that is unattacked. So, $D$ cannot defend $\tuple{x,y}$, contradicting $D$ being a complete extension. 
Second, $\gamma$ is self-attacked. In this case, $(\o,\gamma)\in \und{\mathcal{F}} $, and $\attackd{(\o,\gamma)}{ \tuple{x,y}}$. Since $\tuple{x,y}$ is defended by $D$, $\exists \tuple{\eta, \eta^\prime}\in D$ such that $\attack{\eta}{\gamma}$ or $\attack{\eta^\prime}{\gamma}$. In other word, it holds that $\attack{\mathrm{def}(D)}{\gamma}$.
Third,   $\gamma$ is attacked by  $\eta\in \AR\setminus \{\gamma\}$, %. In this case, if $\eta\in \mathrm{def}(D)$, then $\attack{\mathrm{def}(D)}{\gamma}$. Otherwise, 
there are the following situations:
\begin{itemize} 
\item  $\eta$ is initial or all attackers of $\eta$ are attacked by $ \mathrm{def}(D)$: In this case, $\eta$ does not attack $x$ or $y$. Otherwise, $\tuple{x, y}\notin D$. Contradiction. Meanwhile, since  $\{\tuple{\eta, x}\}\cup D$ (reps., $\{\tuple{\eta, y}\}\cup D$) is conflict-free, and $D$ defends $\tuple{\eta, x}$ (reps., $\{\tuple{\eta, y}\}$). Since $D$ is complete, $\tuple{\eta, x}\in D$ (reps., $\tuple{\eta, y}\in D$).  So, it holds that $\attack{\mathrm{def}(D)}{\gamma}$.%, contradicting $\mathrm{def}(D)\nrightarrow \gamma$.
\item $\eta$ is self-attacked. In this case, $(\o,\gamma)\in \und{\mathcal{F}} $. According to the second point above, it holds that $\attack{\mathrm{def}(D)}{\gamma}$.
\item  $\eta$ is attacked by  $\eta^\prime\in \AR\setminus\{\eta\}$ such that $\eta^\prime$ is not attacked by $ \mathrm{def}(D)$: In this case, $(\eta^\prime, \gamma)\in \dgn{\mathcal{F}} $, and $\attackd{(\eta^\prime, \gamma)}{\tuple{x, y}}$. Since $\tuple{x,y}$ is defended by $D$, $\exists \tuple{\theta, \theta^\prime}\in D$ such that $\attackd{\tuple{\theta, \theta^\prime}}{(\eta^\prime, \gamma)}$. Since $\eta^\prime$ is not attacked by  $\theta$ or $ \theta^\prime$, $\gamma$ is attacked by $\theta$ or $ \theta^\prime$. In other words, it holds that $\attack{\mathrm{def}(D)}{\gamma}$.
\end{itemize}
%Therefore, we can summarize that for all $D\in \mathrm{co}(\dg{\mathcal{F}})$,   if $\tuple{\alpha,\beta}\in D$ then $\tuple{\o,\alpha}\in D$ or $\exists\eta\in \AR$ such that $\tuple{\eta, \alpha}\in D$. 
\end{proof}

\noindent
2. Proof of Theorem \ref{prop-ad}
 
\begin{proof}
%For all $E\in   \mathrm{co}(\dg{\mathcal{F}})^\mathcal{JUS}$, $E\in \mathrm{co}(\mathcal{F})$. Here, $E = \mathrm{def}(D)$ where $D\in  \mathrm{co}(\dg{\mathcal{F}})$. 
Let $E = \mathrm{def}(D)$.  Under complete semantics, we need to prove: 1) $E$ is conflict-free, 2) $E$ defends each member of $E$, and 3) each argument in $\AR$ that is defended by $E$ is in $E$. Details:
\begin{itemize}
\item For all $\alpha, \beta  \in E$, $\alpha$ and $ \beta$ are defenders or defendees of defenses in $D$. Since $D$ is conflict-free, according to Definition \ref{def-avdadg}, it is obvious that $E$ is conflict-free.

\item  For all $\alpha\in E$, $\exists\tuple{x, \alpha}\in D$ or $\tuple{\alpha, x}\in D$ where $x\in \AR\cup\{\o\}$. For all $\gamma\in \AR$, if $\gamma$ attacks $\alpha$,  according to Lemma \ref{lem-1},  $\attack{E}{\gamma}$. So, $\alpha$ is defended by $E$.

\item For all $\alpha\in \AR$, if $\alpha$ is defended by $E$, we have the following possible cases:
\begin{itemize}
\item  $\alpha$ is unattacked: In this case, $\tuple{\o, \alpha}$ is in $D$. That is, $\alpha\in E$.

\item $\alpha$ is attacked by some arguments in $\mathcal{F}$:  For all $\gamma\in \attackers{\alpha}$, since $\alpha$ is defended by $E$, there exists $\delta\in E$ such that $\attack{\delta}{\gamma}$.  It follows that $\tuple{\delta,\alpha}\in \nmd{\mathcal{F}} $, and $\exists\tuple{x, \delta}\in D$ or $\tuple{\delta, x}\in D$ where $x\in E\cup\{\o\}$. . %, and $\exists\tuple{\o, \beta}\in D$. % or $\tuple{\eta, \beta}\in D$ or  $\tuple{\beta, \eta^\prime}\in D$ where $\eta, \eta^\prime\in \AR$. 
Then, we have the following:
\begin{itemize}
\item if $\tuple{\delta,\alpha}$ is unattacked, then since $D$ is complete, $\tuple{\delta,\alpha}\in D$, i.e., $\alpha\in E$; otherwise,
\item for all $[u, \eta]\in \dgn{\mathcal{F}} $: $\attackd{[u, \eta]}{\tuple{\delta,\alpha}}$, if $u$ or $\eta$ attacks $\alpha$, then since $\alpha$ is defended by $E$, there exists $\tuple{\psi, \psi^\prime}\in D$ such that $\psi$ or $\psi^\prime$ attacks $u$ or $\eta$. In other words, $\tuple{\psi, \psi^\prime}$ attacks $[u, \eta]$; if $u$ or $\eta$ attacks $\delta$, then $[u, \eta]$ attacks $\tuple{x, \delta}$ or $\tuple{\delta, x}$. Since $D$ is complete, there exists $\tuple{\theta, \theta^\prime}\in D$ such that $\tuple{\theta, \theta^\prime}$ attacks $[u, \eta]$. So, $\tuple{\delta, \alpha}$ is defended by $D$. Since $D$ is complete, it holds that $\tuple{\delta, \alpha}\in D$, and therefore $\alpha\in E$.
 \end{itemize}
\end{itemize}
%Thus, in all cases, if $\alpha$ is defended by $E$ , then $\alpha\in E$.
\end{itemize}

%So, we may conclude that $E\in \mathrm{co}(\mathcal{F})$.
%Second, under ground semantics, we need to verify that $E$ is minimal (w.r.t. set inclusion). Assume the contrary. Then $\exists E^\prime \subsetneq E$ such that $E^\prime$ is a grounded extension. 
\end{proof}

\noindent
3. Proof of Theorem \ref{prop-adco}

\begin{proof}
For all $E\in \mathrm{co}(\mathcal{F})$, since it is obvious that $\mathrm{d}(E)$ is conflict-free, we need to verify: 1) $\mathrm{d}(E)$ defends each member of $\mathrm{d}(E)$, and 2) each defense in $\nmd{\mathcal{F}} $ that is defended by $\mathrm{d}(E)$ is in $\mathrm{d}(E)$. Details: 
\begin{itemize}
\item For all $\tuple{\beta,\alpha}\in \mathrm{d}(E)$, for all $[x, y]\in \dgn{\mathcal{F}} $, if $[x, y]$ attacks $\tuple{\beta,\alpha}$ such that $\attack{x}{\beta}$ or $\attack{y}{\beta}$, or $\attack{x}{\alpha}$ or $\attack{y}{\alpha}$, since $E$ is a complete extension, $\exists\eta\in E$ such that $\attack{\eta}{x}$ or $\attack{\eta}{y}$. So, $\tuple{\eta,\beta}$ or $\tuple{\eta,\alpha}$ is in $\mathrm{d}(E)$, and $\tuple{\eta,\beta}$ or $\tuple{\eta,\alpha}$ attacks $[x,y]$. %1a) If $\eta$ is unattacked, then $\tuple{\o, \eta}\in D$ and $\tuple{\o, \eta}$ attacks $\tuple{x, y}$. 1b) If $\eta$ is attacked by $z\in \AR$, since $E$ is a complete extension, $\exists\eta^\prime\in E$ such that $\attack{\eta^\prime}{z}$.  It follows that $\tuple{\eta^\prime, \eta}\in D$ and $\tuple{\eta^\prime, \eta}$ attacks $\tuple{x, y}$.
%For all $\tuple{\o,\alpha}\in \mathrm{d}(E)$ and $\tuple{\alpha,\alpha}\in \mathrm{d}(E)$, we also have the same result.  
In other words, $\mathrm{d}(E)$ defends each member of $\mathrm{d}(E)$.

\item For all $\tuple{\alpha, \beta}\in \nmd{\mathcal{F}} $, if $\tuple{\alpha, \beta}$ is defended by $\mathrm{d}(E)$, then both $\alpha$ and $\beta$ are defended by $\mathrm{def}(\mathrm{d}(E)) = E$. Since $E$ is a complete extension,  $\alpha,\beta\in E$. So, $\tuple{\alpha, \beta}\in \mathrm{d}(E)$.   
\end{itemize}
\end{proof}

\noindent
4. Proof of Corollary \ref{cor-com}

\begin{proof}
For $E = \mathrm{def}(D)$, under grounded semantics, we need to prove that $E$ is minimal (w.r.t. set-inclusion). Assume the contrary. Then $\exists E^\prime \subsetneq E$ such that $E^\prime$ is a grounded extension. According to Theorem \ref{prop-adco}, it holds that $\mathrm{d}(E)\in \mathrm{CO}(\dg{\mathcal{F}})$ and $ \mathrm{d}(E^\prime) \in \mathrm{CO}(\dg{\mathcal{F}})$. Since $E^\prime \subsetneq E$, it holds that $\mathrm{d}(E^\prime) \subsetneq \mathrm{d}(E)$. Since $\mathrm{d}(E) = \mathrm{d}(\mathrm{def}(D)) = D$, $\mathrm{d}(E^\prime) \subsetneq D$. It turns out that $D$ is not a minimal complete extension, contradicting $D \in \mathrm{GR}(\dg{\mathcal{F}})$. 

Under preferred semantic, similarly, it is easy to verify that $E$ is maximal (w.r.t. set-inclusion).  So, for all $D \in \mathrm{PR}(\dg{\mathcal{F}})$,  $\mathrm{def}(D) \in \mathrm{pr}(\mathcal{F})$.

Under stable semantics, we need to prove that for all $\alpha\in \AR\setminus E$: $\attack{E}{\alpha}$. Assume the contrary. Then,  $\exists\alpha\in \AR\setminus E$ such that ${E}$ does not attack ${\alpha}$. There are the following possible cases:
\begin{itemize}
\item  $\alpha$ self-attacks. In this case, $(\o, \alpha)\in \dgn{\mathcal{F}} \setminus D$. Since $D$ is a stable extension, $D$ attacks $(\o, \alpha)$. So, $E$ attacks $\alpha$. Contradiction. 
\item $\alpha$ does not self-attack. Since $\alpha$ can not be initial, $\alpha$ is attacked by some argument $\beta\in \AR$. It follows that $\beta\notin E$ and $\beta$ does not self-attack. So, $\beta$ is attacked by some argument $\gamma\in \AR$. So, $[\gamma, \alpha]\in \dgn{\mathcal{F}} $. Since $\alpha\notin E$, $[\gamma, \alpha]\notin D$. Since $D$ is a stable extension, $D$ attacks $[\gamma, \alpha]$. Since $E = \mathrm{def}(D)$ does not attack $\alpha$, $\exists\eta\in E$ such that $\eta$ attacks $\gamma$, and $\eta$ does not attack $\alpha$. Since $\alpha, \beta$ and $\eta$ do not self-attack, we have the following possible cases:
\begin{itemize}
\item $\{\eta,\beta\}$ is conflict-free: In this case, $\tuple{\eta,\beta}\in \nmd{\mathcal{F}} $. Since $E$ cannot attack $\eta$, if $\tuple{\eta,\beta}\notin D$, $\exists\psi\in E$ such that $\psi$ attacks $\beta$. So, $[\psi, \alpha]\in \dgn{\mathcal{F}} $. Since $[\psi, \alpha]\notin D$, $[\psi, \alpha]$ is attacked by $D$. Since $E$ does not attack $\psi$, $E$ attacks $\alpha$. Contradiction. 
\item $\{\eta,\beta\}$ is not conflict-free: If $\eta$ attacks $\beta$, $[\eta, \alpha]\in \dgn{\mathcal{F}} $. It follows that $E$ attacks $\alpha$. Contradiction. If $\eta$ does not attack  $\beta$,  but $\beta$ attacks $\eta$, $\tuple{\eta,\beta}\in \nmd{\mathcal{F}} $. This case also leads to a contradiction.
\end{itemize}
\end{itemize}
\end{proof}

\end{document}